\documentclass[10pt,twocolumn,letterpaper]{article}

\usepackage{iccv}
\usepackage{times}
\usepackage{epsfig}
\usepackage{graphicx}
\usepackage{amsmath}
\usepackage{amssymb}
\usepackage{mathtools}
\usepackage{booktabs}
\usepackage{amsthm}
\usepackage[ruled]{algorithm2e}
\usepackage[usenames,dvipsnames]{xcolor}

\usepackage{fancyhdr}
\usepackage{mathtools}

\newcommand{\mbb}{\mathbf{b}}

\newcommand{\mbf}{\mathbf{f}}
\newcommand{\mbg}{\mathbf{g}}

\newcommand{\mbx}{\mathbf{x}}
\newcommand{\mby}{\mathbf{y}}

\newcommand{\mbD}{\mathbf{D}}

\newcommand{\R}{\mathbb{R}}

\newcommand{\calC}{\mathcal{C}}
\newcommand{\calD}{\mathcal{D}}
\newcommand{\calE}{\mathcal{E}}

\newcommand{\calK}{\mathcal{K}}

\newcommand{\calR}{\mathcal{R}}

\newcommand{\calT}{\mathcal{T}}

\newcommand{\calX}{\mathcal{X}}

\newcommand{\inner}[1]{\left\langle#1\right\rangle}

\def\R{\mathbb{R}}

\def\calX{\mathcal{X}}

\def\det{\mathop{\rm det}\nolimits}

\def\argmax{\mathop{\rm arg\,max}\limits}

\def\ones{\mathbf{1}}

\newtheorem{proposition}{Proposition}

\newcommand{\ignore}[1]{}

\makeatletter
\DeclareRobustCommand\onedot{\futurelet\@let@token\@onedot}
\def\@onedot{\ifx\@let@token.\else.\null\fi\xspace}
\def\eg{{e.g}\onedot} 
\def\ie{{i.e}\onedot} 
\def\cf{{cf}\onedot}

\def\etal{\emph{et al}\onedot}
\makeatother

\expandafter\def\expandafter\normalsize\expandafter{%
\normalsize\setlength\abovedisplayskip{4pt}}

\expandafter\def\expandafter\normalsize\expandafter{%
\normalsize\setlength\belowdisplayskip{4pt}}

\newcommand{\ReRankOne}[1]{\color{BlueViolet}\textbf{{#1}}}
\newcommand{\ReRankTwo}[1]{\color{PineGreen}{#1}}
\newcommand{\ReRankThree}[1]{\color{RedOrange}{#1}}

\usepackage[pagebackref=true,breaklinks=true,letterpaper=true,colorlinks,bookmarks=false]{hyperref}

\iccvfinalcopy 
\ificcvfinal\fi

\begin{document}
\title{Context-guided diffusion for label propagation on graphs}

\author{Kwang In Kim\\
\normalsize{Lancaster University}\\
\and
James Tompkin\\
\normalsize{Harvard Paulson SEAS}\\
\and
Hanspeter Pfister\\
\normalsize{Harvard Paulson SEAS}\\
\and
Christian Theobalt\\
\normalsize{MPI for Informatics}
}

\maketitle

\begin{abstract} 
Existing approaches for diffusion on graphs, \eg, for label propagation, are mainly focused on isotropic diffusion, which is induced by the commonly-used graph Laplacian regularizer. Inspired by the success of diffusivity tensors for anisotropic diffusion in image processing, we presents anisotropic diffusion on graphs and the corresponding label propagation algorithm. We develop positive definite diffusivity operators on the vector bundles of Riemannian manifolds, and discretize them to diffusivity operators on graphs. This enables us to easily define new robust diffusivity operators which significantly improve semi-supervised learning performance over existing diffusion algorithms.
\end{abstract}

\thispagestyle{fancy}

\section{Introduction}
Physical diffusion describes how energy, mass, or substances spread over time --- how their densities \emph{smoothen out} in a medium. Simulating physical diffusion on a Euclidean space, a manifold, or their discrete approximations, \eg, grids or graphs, has application in image processing, computer vision, and machine learning. For instance, diffusion is now a standard tool for removing noise or to highlight salient structures~\cite{Wei98}. The graph Laplacian, as a discrete approximation of the generator of the diffusion process on manifolds, i.e., the Laplace-Beltrami operator, is commonly used in spectral clustering and semi-supervised learning, which finds applications in object recognition~\cite{EbeFriSch12-1,WuYuWan13}, image retrieval~\cite{ElbWerHel13}, and segmentation and matting~\cite{CasNonTau14,LiCheTan13}. Similarly, stochastic diffusion process on graphs find application in multi-label classification~\cite{WanTuTso13} and image retrieval~\cite{GopTurChe10}.

In these applications, typically we are given a set of objects $X=\{\mbx_1,\ldots,\mbx_n\}$ and corresponding assignments of variables $Y^t=\{\mby^t_1,\ldots,\mby^t_n\}$ at time $t=0$. Then, (simulated) diffusion models how $Y$ \emph{smooths} over $X$. For instance, when $X$ denotes vertices of a mesh, $Y$ is the coordinate representations of $X$ in an embedding space $\calX$, leading to mesh fairing. More generally, if $X$ denotes noisy observations of data points lying on a manifold, diffusion leads to manifold denoising. If $Y$ represents class labels of data points in $X$, diffusion leads to label propagation and facilitates semi-supervised learning. In this case, $Y$ is assumed to be a sample from an underlying classification function $f$ on $\calX$ (\ie, $Y=\{\mby_1,\ldots,\mby_n\}=\{f(\mbx_1),\ldots,f(\mbx_n)\}$). 

Diffusion is determined by the initial condition $Y^0$ and the \emph{diffusivity} defined on $X$ or $\calX$. Roughly, the diffusivity describes the direction and strength of $f$ (and equivalently $Y$) being smoothed at each time instance $t$. In general, the diffusivity is inhomogeneous as it varies over $X$, and is anisotropic as its strength varies over different directions at each point $\mbx\in X$. For instance, in image processing, diffusivity is strong in flat regions but weaker on edges. Further, on an edge, diffusivity is stronger along the direction of edges than across it. This leads to edge-preserving image smoothing as pioneered by Weickert~\cite{Wei98}. 

For graph data, diffusion can be seen as label propagation in semi-supervised learning. Thus far, label propagation has mainly focused on isotropic diffusion (\ie, the diffusivity is fixed on the entire data space and all directions at each point therein), and only recently has anisotropic diffusion been explored: Coifman and Lafon~\cite{CofLaf06} apply anisotropic diffusion to the graph-based dimensionality reduction problem. They control diffusivity by normalizing the (originally isotropic) pair-wise similarity with the evaluations of diffused coordinate values. Szlam~\etal~\cite{SzlMagCoi08} generalizes and extends this framework to semi-supervised learning by controlling diffusivity via evaluations of class labels $f$: If $f(\mbx_i)$ and $f(\mbx_j)$ are similar, \ie, if the class labels of $\mbx_i$ and $\mbx_j$ are likely to be the same, then diffusivity along the edge joining them is high. Otherwise, diffusivity becomes low, which prevents label propagation across class boundaries. This leads to significant performance improvement over classical isotropic diffusion. Kim~\etal~\cite{KimTomThe13} proposed adapting diffusivity on Riemannian manifolds based on local curvature estimates: Diffusivity is strong in flat regions and weak along the direction of the curvature operator, which leads to an awareness of intersections between manifolds and so improves performance over isotropic equivalents. However, this requires the data $X$ to be embedded in an ambient Euclidean space, and so does not apply to inference on general graphs. 

We propose two contributions for anisotropic diffusion on graphs. First, we analyze continuous anisotropic diffusion processes on smooth manifolds, and show that anisotropic diffusion is nothing more than isotropic diffusion on a manifold with a new metric. Based on this analysis, we arrive at a new anisotropic graph Laplacian approach which is similar to the stochastic kernel smoothing approach of Szlam~\etal~\cite{SzlMagCoi08}, but with a new geometric intuition. This provides explicit criteria to define valid diffusivities on graphs and manifolds, and it facilitates non-linear diffusion on graphs. Second, we explore two possible operators which control diffusivity of each edge based on local neighborhood contexts and not just their end vertices. This \emph{context-guided diffusion} extends to graphs the robust diffusion algorithm originally developed for image enhancement~\cite{Wei98}, and we demonstrate on 11 different classification problems that this improves semi-supervised learning performance over isotropic diffusion, the stochastic anisotropic diffusion of Szlam~\etal~\cite{SzlMagCoi08}, and three existing label propagation algorithms~\cite{ZhuGhaLaf03,GonTao14,WanZha06}.

To assist readers and subsequent development, we make our code available on the web.

\section{Anisotropic diffusion on graphs}
\label{s:anidiffgraph}
We develop anisotropic analogs to the existing isotropic diffusion process and to the corresponding graph Laplacian. We also introduce context-guided diffusion for semi-supervised learning. These contributions are based on the analysis of the continuous positive definite diffusivity operators on Riemannian manifolds, which we leave for Sec.~\ref{s:continuouscase}.

Existing works \cite{ZhoSch06,HeiAudLix05} establish the (isotropic) graph Laplacian as a discrete approximation of the Laplace-Beltrami operator on a data manifold. We build upon these works to develop isotropic and anisotropic graph Laplacians by combining local diffusivity operators defined on sub-graphs centered at each data point. As such, first, we explain existing approaches.

\paragraph{Discrete isotropic diffusion.}
A weighted graph $(X,E,W)$ consists of sets of nodes $X$ of size $n$
, edges $E\subset X\times X$, and non-negative similarities $w_{ij}:=w(e_{ij})\in W$ for each edge $e_{ij}\in E$, with $w_{ij}=0$ if $e_{ij}\notin E$.

For subsequent definition of diffusivity operators based on local gradients and divergences, we need spaces with defined inner products (\ie, Hilbert spaces), and so we introduce spaces $H(X)$ and $H(E)$ of functions on $X$ and $E$, with inner products defined as~\cite{ZhoSch06,HeiAudLix05}:
\begin{align}
\inner{f,h}_{H(X)} &= \sum_{i=1}^nf(i)h(i) d_i, \forall f, g \in H(X),\\
\inner{S,T}_{H(E)} &= \sum_{i,j=1}^n S(i,j)T(i,j), \forall S, T \in H(E),
\end{align}
where $f(i)=f(\mbx_i)$ and $d_i$ is the degree of node $\mbx_i\in X$:
\begin{align}
\label{e:nodedegree}
d_i = \sum_{j=1}^n w_{ij}.
\end{align}

For each node $\mbx_i$, a subgraph $G_i=(X_i,E_i,W_i)$ centered at $\mbx_i$ is defined as the set of nodes that are connected to $\mbx_i$ and the corresponding edges, \ie, $X_i=\{\mbx_j| e_{ij}\in E\}$, $E_i=\{e_{ij}|\mbx_j\in X_i\}$, and $W_i$ are obtained by evaluating $W$ at $E_i$. 
 The inner-product structures on $X_i$ and $E_i$ are induced as restrictions of the corresponding structures on the entire graph $G$ to the sub-graph $G_i$, which we denote by $H(X_i)$ and $H(E_i)$, respectively. Given these structures, we define discrete gradient and divergence operators at $G_i$. First, the graph gradient operator $\nabla_i:H(X_i)\to H(E_i)$ is defined as the collection of $f$ differences along the edges:
\begin{align}
[\nabla_i f](e_{ij}) = \sqrt{w_{ij}}(f(j)-f(i)),
\label{e:graddef}
\end{align}
for $e_{ij}\in E_i$ and $f\in H(X_i)$. The graph divergence operator $\nabla_i^*:H(E_i)\to H(X_i)$ is defined as the formal adjoint of $\nabla_i$: for all $f \in H(X_i), S\in H(E_i)$:
\begin{align}
\inner{\nabla_i f,S}_{H(E_i)} = \inner{f,\nabla_i^*S}_{H(X_i)}.
\label{e:divdef}
\end{align}

By substituting Eq.~\ref{e:graddef} into Eq.~\ref{e:divdef}, $\nabla_i^*$ is explicitly given as
\begin{align}
[\nabla_i^*S](i) = \frac{1}{2d_i}\sum_{j=1}^n\sqrt{w_{ji}}(S(j,i)-S(i,j)).
\end{align}

By combining the local gradient and divergence operators, we can construct the global normalized graph Laplacian $L:H(X)\to H(X)$:
\begin{align}
	[Lf](i)=\nabla_i^*\nabla_i f, \text{ }\forall f \in H(X), i=1,\ldots,n.
	\label{e:graphlaplocal}
\end{align}
Our definition of the graph Laplacian is consistent with~\cite{ZhoSch06,HeiAudLix05}. In particular, at the $i$-th node, it is explicitly given as:
\begin{align}
[Lf](i) & = f(i)-\frac{1}{d_i}\sum_{j=1}^n w_{ji}f(j).
\end{align}
If the nodes $X$ of $G$ are sampled from an underlying data generating manifold $M$, \ie, the probability distribution $P(\mbx)$ is supported in $M$, the graph Laplacian $L$ converges to the Laplace-Beltrami operator $\Delta$ on $M$ as $n\to\infty$~\cite{HeiAudLix05,BelNiy05}. This is often regarded as the reason for using graph Laplacian as a regularizer in many applications: The semi-norm $\|f\|_\Delta$ induced by $\Delta$ is equivalent to the norm of the gradient $\nabla f$ of a function $f$ on $M$ (see Sec.~\ref{s:continuouscase}). Then, $Lf$ is obtained as a discrete approximation of the first-order regularizer on graphs. Further, $\Delta$ is the generator of isotropic diffusion process on $M$ and accordingly, $L$ is also a discrete approximation of the isotropic diffusion generator on $G$.

\paragraph{Anisotropic diffusion on graphs.}
Next, we extend isotropic graph Laplacian $L$ to be anisotropic. Our derivation is based on Weickert's definition on positive definite (PD) diffusivity operators on $\R^2$~\cite{Wei98}. In Section~\ref{s:continuouscase}, we introduce an extension of these operators to general Riemannian manifolds and, based on that, establish a rigorous connection between our anisotropic diffusion process on $G$ and that of the data generating manifold $M$.

First, we formally introduce the local diffusivity operator $D_i:H(E_i)\to H(E_i)$:
\begin{align}
D_i &:= \sum_{j\neq i,\mbx_j\in X_i}q_{ij}\mbb_{ij}\otimes \mbb_{ij}\nonumber\\
\Leftrightarrow [D_i S](e_{ij}) &= q_{ij}\mbb_{ij}\inner{\mbb_{ij},S}, \forall S \in H(E_i),
\label{e:discdiffusivity}
\end{align}
where $\otimes$ is the tensor product and the \emph{basis function} $\mbb_{ij}$ is defined as the indicator of $e_{ij}$, \ie, $\mbb_{ij}=\ones_{ij}$. Similar to the construction of diffusivity operators on $\R^2$~\cite{Wei98}, our diffusivity operators are constructed based on its spectral decomposition: $q_{ij}$ is an eigenvalue of the operator $D_i$ corresponding to the eigenfunction $\mbb_{ij}$. This enables us to straightforwardly define a globally PD diffusivity operator on $G$: Our global diffusivity operator $D:H(E)\to H(E)$ is obtained by identifying $D_i$ as the restriction of $D$ on $H(E_i)$. In this case, $D$ is positive definite if and only if $\{q_{ij}\}$ is symmetric and positive, \ie, $q_{jk}=q_{kj},q_{jk}> 0,\forall j,k=1,\ldots,n$. Furthermore, $D$ is \emph{uniformly PD} if all eigenvalues $\{q_{ij}\}$ are lower-bounded by a positive constant $\nu$. 

Now we are ready to define an anisotropic diffusion process on $G$. We construct an anisotropic graph Laplacian:
\begin{align}
[L^D f](i) &:= [\nabla^*_i D_i \nabla_i f](i),\nonumber\\
&=\left(\frac{1}{d_i}\sum_{j=1}^n w_{ij}q_{ij}\right)f(i)-\frac{1}{d_i}\sum_{j=1}^n w_{ij}q_{ij}f(j),
\label{e:aniglap}
\end{align}
where the equality in the second line is obtained by substituting Eqs.~\ref{e:graddef}, \ref{e:divdef}, and \ref{e:discdiffusivity} into the first line.

Except for the normalization term in $f(i)$, the construction of $L^D$ is identical to the isotropic graph Laplacian $L$ case: The original weights $\{w_{ij}\}$ are replaced by new weights $\{w^D_{ij}\}$:
\begin{align}
\label{e:aniw}
	w^D_{ij} = w_{ij}q_{ij}.
\end{align}

Given the anisotropic graph Laplacian $L_D$, we can define the corresponding anisotropic diffusion process on $G$. For instance, for label propagation applications, we propose using the explicit Euler approximation (\cf~Eq.~\ref{e:anilap} for the continuous counterpart):
\begin{align}
\frac{f^{t+1}-f^{t}}{\delta} &= -L^D f^{t}\nonumber\\
\Leftrightarrow f^{t+1} &= f^{t}-\delta L^D f^{t},
\label{e:graphanidiff}
\end{align}
where $f^t$ denotes the value of $f$ at time $t$ and $\delta$ is the time discretization interval.
The uniform positive definiteness of the diffusivity operators is crucial to the well-posedness of the corresponding diffusion process in $\R^2$~\cite{Wei98}. The same applies to the positive definiteness of our discrete diffusivity operator $D$: This is the only way that $L_D$ is a conditionally PD matrix and therefore it can be a valid regularizer on $G$: 
\begin{align}
R_{L^D}(f):=\mbf^\top L^D \mbf=\sum_{i,j=1,\ldots,n} w^D_{ij}/d_i\left(f(i)-f(j)\right)^2,
\label{e:graphregul}
\end{align}
where $\mbf=[f(1),\ldots,f(n)]^\top$: For simplicity, we assume that $f(i)$ is a scalar. When $f(i)$ is a vector, \eg, for multi-class classification, $R_{L^D}(f)$ is summed over the output dimensions. If $D$ is fixed throughout diffusion, the difference equation~(\ref{e:graphanidiff}) is linear and the corresponding analytical solution $f^t$ exists for any $\delta>0$ and $t>0$ given $f^0$. However, in general, $D$ depends on $f^t$ (\eg, Eq.~\ref{e:outputkernel}) and so Eq.~\ref{e:graphanidiff} becomes nonlinear, where the solution $f^t$ can be obtained by iterating updating $f^t$ with the right side of Eq.~\ref{e:graphanidiff}.

\paragraph{Anisotropic diffusion for semi-supervised learning.}
With proper choices of $\{q_{ij}\}$, our diffusion equation (Eq.~\ref{e:graphanidiff}) can be used in various applications including label propagation for semi-supervised learning. Assume we are given a set of data points $X=\{\mbx_1,\ldots,\mbx_n\}\in \R^d$ where only the first $l$-data points are provided with the ground-truth class labels $Y=\{\mby_1,\ldots,\mby_l\}$. Our goal is to \emph{propagate} these labels to the entire dataset $X$. We approach this problem by first building a graph $G=(X,E,W)$ with:
\begin{align}
\label{e:isoweights}
w_{jk} = \left\{
\begin{array}{c l}      
    \exp\left(-\frac{\|\mbx_j-\mbx_k\|^2}{\sigma_\mbx}\right) & \text{if } \mbx_j\in N_K(\mbx_k)\\
		&\text{ or } \mbx_k\in N_K(\mbx_j) \\
    0 & \text{otherwise},
\end{array}\right.
\end{align}
where $N_K(\mbx_j)$ is the $K$-nearest neighborhood of $\mbx_j$ and $\sigma_\mbx>0$ is a hyper-parameter. Then, we diffuse the labels $Y$ on $G$. Specifically, our label propagation algorithm adopts the approach of Zhou~\etal~\cite{ZhoBouLal03}: For a $c$-class classification problem, each label $\mby_k\in Y$ is given as a $c$-dimensional row vector. When the ground-truth class of $\mbx_j$ is $k$, the elements of $\mby_j$ are all zero except for the $k$-th element that is assigned with one: $\mby_j=[0,\ldots,1,\ldots,0]$. The label propagation is then performed by building the initial $f^0\in \R^{n \times c}$ where $i$-th row is $\mby_i$ if $\mbx_i$ is labeled ($i\leq l$) and $0$, otherwise, and running the difference equation (explicit Euler scheme; Eq.~\ref{e:graphanidiff}) until the stopping criteria is met: As suggested by the form of regularizer $\calR_{L^D}$, similarly to the isotropic graph Laplacian, the only null-space of anisotropic graph Laplacian is the space of constant functions. This implies that the difference equation (Eq.~\ref{e:graphanidiff}) converges to a constant function as $t\to\infty$. Accordingly, for practical applications, we stop diffusion at a finite time step $T$ and obtain the resulting function $f^T$ as the output. The final class label for data point $\mbx_i$ is obtained as $\argmax f^T(i)\in \R^c$ for each $i$.

The best choice for the eigenvalues $\{q_{ij}\}$ of the diffusivity operator $D$ depends on the application. Intuitively, the diffusivity $q_{ij}$ should be high when the corresponding function evaluations $f(i)$ and $f(j)$ are similar, \ie, $|\nabla_i f(e_{ij})|$ is small. One way to define such diffusivity is to use a Gaussian weight function as is common in image enhancement: 
\begin{equation}
q_{ij} = \exp\left(-\frac{|\nabla_i f(e_{ij})|^2}{\sigma^2_f}\right),
\label{e:outputkernel}
\end{equation}
where $\sigma^2_f$ is the scale hyper-parameter. Algorithm \ref{a:mainalg} shows pseudocode to construct the corresponding anisotropic graph Laplacian on $G$.

\begin{algorithm}[t]
\caption{Build anisotropic graph Laplacian $L^D$.}
\SetKwInput{Input}{Input}
\SetKwInput{Output}{Output}
\Input{Set of data points $X=\{\mbx_1,\ldots,\mbx_n\}\subset \R^d$\\with function values: $F=\{f(\mbx_1),\ldots,f(\mbx_n)\}\subset \R^c$.}
\Output{{$L^D$.}}
\BlankLine
\For{$i=1,\ldots,n$}
{
Find nearest neighbors $N_K(\mbx_i)$;\\
Calculate isotropic weights $w_{ij}$ (for $\mbx_j\in N_K(\mbx_i)$ and $\mbx_i\in N_K(\mbx_j)$; Eq.~\ref{e:isoweights});\\
Calculate the node degree $d_{i}$ (Eq.~\ref{e:nodedegree});\\
Calculate the diffusivity eigenvalues $q_{ij}$ using one of Eqs.~\ref{e:outputkernel}, \ref{e:smoothdiff}, and \ref{e:matchdiff};
}
Rearrange $\{w^D_{ij}\}$ (Eq.~\ref{e:aniw}) to a matrix $L^D$ based on Eq.~\ref{e:aniglap}.
\label{a:mainalg}
\end{algorithm}

The resulting anisotropic graph Laplacian $L^D$ can be immediately applied to any label-propagation problems. However, for semi-supervised learning algorithm, na\"ively applying $L^D$ to the difference equation (\ref{e:graphanidiff}) may require many iterations before it actually starts propagating labels. The progress of diffusion can be very slow in the early stage ($t$ is small) at the vicinity of labeled points: If a point $\mbx_i$ is labeled and $N_K(\mbx_i)\backslash \mbx_i$ are all unlabeled (this is typically the case for semi-supervised learning), the corresponding eigenvalues (Eq.~\ref{e:outputkernel}) are all small, and accordingly, the weights $\{w^D_{ij}\}$ are also small for all $\mbx_j\in N_K(\mbx_i)$. To speed up the process, we run the isotropic diffusion (with the isotropic graph Laplacian $L$) and \emph{smooth} out the initial distribution of $f^0$. For all experiments, the initial diffusion runs for 20 time steps while the length $T$ of the anisotropic diffusion is regarded as a hyper-parameter.

\paragraph{Discussion.}
Our derivation of anisotropic graph Laplacian is strongly connected to the kernel-based anisotropic diffusion approach of Szlam~\etal~\cite{SzlMagCoi08}, yet the motivating ideas are different: their anisotropic kernel is based on stochastic Markov diffusion processes on graphs, while our anisotropic graph Laplacian is obtained based on a formulation of geometric diffusion on manifolds: $L^D$ is obtained by extending Weickert's diffusivity operators in $\R^2$~\cite{Wei98} to $M$ and then discretizing it onto a graph $G$ (see Sec.~\ref{s:continuouscase}). 

Since the kernel smoothing corresponds to calculating analytic solution at each time step of diffusion, and our anisotropic weights $\{w^D_{ij}\}$ used in constructing $L^D$ can be regarded as an instance of such kernels, the final diffusion algorithms of Szlam~\etal~\cite{SzlMagCoi08} and ours are very similar when applied to linear diffusion: Kernel smoothing is given by first obtaining the continuous Gaussian smoothing as an analytical solution of the linear diffusion equation, and then discretizing it, while our explicit Euler scheme is obtained by directly discretizing both the manifold and the Laplace-Beltrami operator. In preliminary linear diffusion experiments, minor differences in weights normalization\footnote{In $L^D$, the normalization coefficients $\{d_i\}$ are constructed from $\{w_{ij}\}$ (see Eq.~\ref{e:aniglap}), while the diffusion kernel in Szlam~\etal~\cite{SzlMagCoi08} is normalized so that it leads to a stochastic matrix.} led to only negligible differences in semi-supervised learning performances. 

The major differences between the two diffusion algorithms are that 1) our algorithm is nonlinear, \ie $L^D$ depends on $f^t$ at each time $t$, while the anisotropic kernel of \cite{SzlMagCoi08} is obtained as an analytic solution of linear diffusion equation and therefore is fixed a priori to the entire diffusion process. In our experiments, we demonstrate that extending the approach of Szlam~\etal~\cite{SzlMagCoi08} to non-linear diffusion already significantly improves semi-supervised learning performance. Furthermore, unlike Szlam, 2) our construction explicitly states sufficient conditions ($\{q_{ij}\}$ are symmetric and positive) for the well-posedness of the resulting diffusion on $G$ as a discretization of the underlying manifold. This enables exploring various possibilities of inducing new diffusion on $G$.

\subsection{Context-guided diffusion.}
\label{sec:contextguideddiffusion}
We have seen how defining positive eigenvalues $\{q_{ij}\}$ leads to a PD diffusivity operator $D$ and to the corresponding anisotropic graph Laplacian $L^D$. This can be regarded as updating the similarity measure between data points in $X\subset \R^d$: The isotropic graph Laplacian matrix $L$ is constructed from the positive weights $\{w_{ij}\}$ which are the pair-wise similarities of data points measured by the original Euclidean metric of $\R^d$ (see Eq.~\ref{e:isoweights}). By construction, the information in $L$ is precisely the same as the pair-wise similarities and, therefore, defining a graph Laplacian $L$ corresponds to defining a similarity measure. Now, defining the anisotropic diffusivity operator $L^D$, which is constructed based on the original similarity measure plus the eigenvalues $\{q_{ij}\}$, can be interpreted as introducing a new similarity measure $\{w^D_{ij}\}$ on $G$.\footnote{This intuition holds rigorously on the Laplace-Beltrami operator $\Delta$ on a Riemannian manifold $M$: 1) Indeed, $\Delta$ uniquely defines a Riemannian metric $g$ on $M$~\cite{Ros97} and 2) Section~\ref{s:continuouscase} shows that defining a diffusivity operator $\calD$ on $M$ corresponds to defining the corresponding new metric $\overline{g}$.}

In particular, we have seen how the Gaussian function (Eq.~\ref{e:outputkernel}) measures the deviation between the two function evaluations $f(i)$ and $f(j)$ as each edge $e_{ij}$. This is only an example and there are various possibilities given the positivity constraint. Furthermore, $q_{ij}$ does not have to defend only based on $f(i)$ and $f(j)$ and it can take into account the neighborhood context as well. For instance, spatially smoothing the diffusivity operator, \eg, by convolving it with a Gaussian kernel, leads to much more stable image enhancement than using the original diffusivity operators (which is commonly constructed based on gradient vectors): Theoretically, the smoothing operation guarantees the well-posedness of the resulting diffusion equation even when the corresponding original version is not. From a practical perspective, this operation offers robustness against noise in the image $f$ since the gross effect of smoothing the diffusivity is to take the spatial averaging of the gradients of $f$~\cite{Wei98}. 

The spatial smoothing of the diffusivity operator can be regarded as an instance of controlling the diffusivity based on \emph{local context}. We investigate two possibilities of exploiting this local context. The first case is to adapt the idea of Gaussian smoothing on images to graphs: For a given edge $e_{ij}$ and the corresponding local neighborhoods at each end node, $N_K(\mbx_i)$ and $N_K(\mbx_j)$, the \emph{smooth diffusivity} $w^D_{ij}$ is obtained based on weighted averages of the diffusivities in the mutual neighborhood $N_M(x_i,x_j):=N_K(x_i) \cap N_K(x_j)$.
\begin{align}
w^D_{ij} = \sum_{x_k\in N_M(x_i,x_j)}w_{ij}\left(q_{ij}+q_{ik}q_{kj}\right)/(s^q_i+s^q_j),
\label{e:smoothdiff}
\end{align}
where $s^q_i=\sum_{x_k\in N_K(x_i)} q_{ik} $ and $s^q_j=\sum_{x_k\in N_K(x_j)} q_{kj}$. The interpretation of our smooth diffusivity is straightforwardly transferred from the smooth diffusivity operators in the image domain: The resulting diffusion process is robust against noise in edge weights. 

Another example of exploiting the context is to adopt the intuitive notion of matching between the two entities in context: If a pair of objects $\mbx_i$ and $\mbx_j$ matches, then often spatial neighbors of $\mbx_i$, $\mbx_l\in N_K(\mbx_i)$ have the corresponding matching elements in their neighborhoods $N_K(\mbx_j)$ of $\mbx_j$, \ie, the match of $(\mbx_i,\mbx_j)$ is \emph{supported} if the neighborhoods of $N_K(\mbx_i)$ and $N_K(\mbx_j)$ find matches in each pair of elements. Our \emph{local match diffusivity} is defined as a smooth version of considering this match context:
\begin{align}
w^D_{ij} = w_{ij}q_{ij}\sum_{\mbx_k\in N_K(\mbx_i)}\left(1+q^*_{ik}\right)/(k+1),
\label{e:matchdiff}
\end{align}
where $q^*_{ik}=\max_{\mbx_l\in N_K(\mbx_j)}q_{kl}$. The $\max$ in the definition of $q^*_{ik}$ implies that if there's any entity in $N_K(\mbx_j)$ that matches $\mbx_k$, the corresponding diffusivity between $\mbx_i$ and $\mbx_j$ is supported. The normalization factor $k+1$ is actually obtained as $k+1$ times the maximum possible value of $q_{ij}$ (which corresponds to the \emph{match} case) which is 1 (Eq.~\ref{e:outputkernel}). 

\section{Connection to continuous operators}
\label{s:continuouscase}
As we have seen in Sec.~\ref{sec:contextguideddiffusion}, our anisotropic diffusion process on $G=(X,E,W)$ is nothing more than isotropic diffusion on a new graph $(X,E,W^D)$ (regularization-form definition of $L^D$ in Eq.~\ref{e:graphregul}, and corresponding diffusion process in Eq.~\ref{e:graphanidiff}) --- our (discrete) diffusivity operator $D$ (Eq.~\ref{e:discdiffusivity}) changes the notion of similarity. In this section, first, we show that this intuition applies to the continuous limit case of Laplace-Beltrami operator $\Delta$ on a data generating manifold $M$, \ie, anisotropic diffusion on $M$ is isotropic diffusion with a new metric. Then, we discuss the convergence properties of our anisotropic graph Laplacian to the continuous anisotropic Laplace-Beltrami operator.

\paragraph{Anisotropic diffusion on Riemannian manifolds.}

On a Riemannian manifold $(M,g)$ with $g$ being a Riemannian metric on $M$, the isotropic diffusion of a smooth function $f \in C^\infty(M)$ is described as a partial differential equation:
\begin{align}
\frac{\partial f}{\partial t} = {\nabla^g}^* \nabla^g f= -\Delta^g f,
\label{e:isodiff}
\end{align}
where $\nabla^g f$ is the gradient of $f$, ${\nabla^g}^*$ is the formal adjoint of ${\nabla^g}$, and $\Delta^g$ is the Laplace-Beltrami operator defined by $\Delta^g=-{\nabla^g}^* \nabla^g$.

If we extend Weickert's diffusivity operator originally defined on $\R^2$~\cite{Wei98} to a manifold $M$, then we introduce a smooth positive definite operator $\calD:\calT(M)\to\calT(M)$ with $\calT(M)$ being the tangent bundle of $M$, \ie, $\calD$ is a smooth field of symmetric positive definite operators each defined on a tangent space $T_\mbx(M)\in\calT(M)$ at $\mbx\in M$. The corresponding anisotropic diffusion process is given as:
\begin{align}
\frac{\partial f}{\partial t} = {\nabla^g}^* \calD \nabla^g f.
\label{e:anidiff}
\end{align}

Defining an anisotropic Laplacian operator $\Delta^g_\calD ={\nabla^g}^* \calD \nabla^g$, we restate Eq.~\ref{e:anidiff} similarly to the isotropic case:
\begin{align}
\frac{\partial f}{\partial t} = -\Delta_\calD f.
\label{e:anilap}
\end{align}
We show that our anisotropic diffusion (Eq.~\ref{e:anilap}) boils down to isotropic diffusion on $M$ with a new metric $\overline{g}$:

\begin{proposition}[The equivalence of $\Delta_\calD$ and $\Delta_{\overline{g}}$]
\label{pr:equivalance}
The anisotropic Laplacian operator $\Delta_\calD$ on a compact Riemannian manifold $(M,g)$ is equivalent to the Laplace-Beltrami operator $\Delta_{\overline{g}}$ on $(M,\overline{g})$ with a new metric $\overline{g}$ depending on $\calD$. Specifically, when the diffusivity operator $\calD$ is uniformly positive definite, $\overline{g}$ is explicitly obtained as $c(\mbx)\overline{\mbg}(\mbx)= \mbg(\mbx) \mbD^{-1}(\mbx)$, where $\mbg(\mbx)$, $\overline{\mbg}(\mbx)$, and $\mbD(\mbx)$ are the coordinate representations (matrices) of $g$, $\overline{g}$, and $\calD$ at each point $\mbx$, and $c(\mbx)=\frac{\sqrt{\det{\mbg(\mbx)}}}{\sqrt{\det{\overline{\mbg}(\mbx)}}}$ which is a smooth function on $M$.
\end{proposition}

\begin{proof}
The proof is obtained by applying the techniques developed for analyzing maps between general \emph{weighted manifolds}~\cite{Gri13}. For any function $f,h\in C^\infty(M)$, we have:
\begin{align}
\int f\Delta_\calD h dV  &= \int f {\nabla^g}^* \calD \nabla^g h dV\nonumber\\
&= -\int \inner{\nabla^g f, \calD \nabla^g h}_g dV\nonumber\\
&= -\int df(\calD \nabla^g h) dV,
\end{align}
where $dV$ is the \emph{natural volume element}~\cite{Lee97} corresponding to $g$ ($dV=\sqrt{\det{g}}d\mbx$) and the second equality is obtained by applying the divergence theorem on $(M,g)$. The third equality corresponds to the definition of gradient $\nabla^g$ based on the differential operator $d$~\cite{Lee97}. Applying Green's theorem to $(M,\overline{g})$, we obtain:
\begin{align}
\int f \Delta_{\overline{g}} h d\overline{V} &= -\int \inner{\nabla^{\overline{g}}f,\nabla^{\overline{g}}h}_{\overline{g}} d\overline{V}\nonumber\\
&=-\int \inner{\nabla^{\overline{g}}f,\frac{\sqrt{\det{\overline{g}}}}{\sqrt{\det{g}}}\nabla^{\overline{g}}h}_{\overline{g}} dV\nonumber\\
&=-\int df\left(\frac{\sqrt{\det{\overline{g}}}}{\sqrt{\det{g}}}\nabla^{\overline{g}}h\right) dV.
\end{align}

Now, identifying the two integrals, and using $\nabla^g h = g^{-1}d h$ and $\nabla^{\overline{g}} h = {\overline{g}}^{-1}d h$, we obtain
\begin{align}
\frac{\sqrt{\det{\overline{\mbg}(\mbx)}}}{\sqrt{\det{\mbg(\mbx)}}}\overline{\mbg}^{-1}(\mbx)&=\mbD(\mbx) \mbg^{-1}(\mbx) \\
\therefore \hspace{0.5cm}  c(\mbx)\overline{\mbg}(\mbx) &= \mbg(\mbx) \mbD^{-1}(\mbx).
\label{e:animetric}
\end{align}
\end{proof}

\ignore{
$\overline{\mbg}_x=\mbg_x\mbD_x^{-1}$.
\begin{align}
\int f \Delta_{\overline{g}} h d\overline{V} &= -\int \inner{\nabla^{\overline{g}}f,\nabla^{\overline{g}}h}_{\overline{g}} \sqrt{\det{\overline{\mbg}_x}}d\lambda\nonumber\\
&=-\int \inner{\nabla^{\overline{g}}f,\nabla^{\overline{g}}h}_{\overline{g}}\sqrt{\det{\mbg_x}}\sqrt{\det{\mbD_x^{-1}}}d\lambda\nonumber\\
&=-\int \inner{\nabla^{\overline{g}}f,\nabla^{\overline{g}}h}_{\overline{g}}\sqrt{\det{\mbD_x^{-1}}}dV\nonumber\\
&=-\int \inner{\nabla^{\overline{g}}f,\sqrt{\det{\mbD_x^{-1}}}\nabla^{\overline{g}}h}_{\overline{g}} dV\nonumber\\
&=-\int df\left(\sqrt{\det{\mbD_x^{-1}}}\nabla^{\overline{g}}h\right) dV.
\end{align}
}

It is always possible to find a coordinate representation of the Riemannian metric $g$ at each point $\mbx\in M$ such that it becomes Euclidean (up to second order)~\cite{Jos11}. This implies that, up to scale,\footnote{Note that the ratio $\frac{\sqrt{\det{\overline{g}}}}{\sqrt{\det{g}}}$ is coordinate independent.} the metric $\overline{g}(\mbx)$ in Eq.~\ref{e:animetric} boils down to well-established Mahalanobis distance, with $\mbD(\mbx)$ being the corresponding covariance matrix in $T_\mbx(M)$. This greatly helps to understand of the anisotropic diffusion process. For any PD diffusivity operator $\calD$, there is a corresponding isotropic Laplace-Beltrami operator $\Delta_{\overline{g}}$ on $(M,\overline{g})$. If we discretize in time the differential equation of the isotropic diffusion process (Eq.~\ref{e:isodiff}) on $(M,\overline{g})$ (see \cite{HeiMai07} for derivation):
\begin{align}
\frac{f^{t+\delta}-f^t}{\delta} = -\Delta_{\overline{g}}f^{t+\delta},
\label{e:impeuler}
\end{align}
then the solution $f^{t+\delta}$ at time $t+\delta$, is obtained as the minimizer of the following regularization energy:\footnote{This applies even when Eq.~\ref{e:impeuler} is nonlinear, \ie $\calD$ depends on $f$.}
\begin{align}
\calE(f) = \|f-f^t\|^2+\delta \int \|\nabla^{\overline{g}} f \|_{\overline{g}} d\overline{V},
\end{align}
which is now equivalent to:
\begin{align}
\calE(f) = \|f-f^t\|^2+\delta \int c\inner{\nabla^{g} f ,\mbD^{-1}\nabla^{g} f}_{g} dV.
\label{e:anienergy}
\end{align}
Accordingly, the anisotropic diffusion process (Eq.~\ref{e:anidiff}) can be regarded as continuously solving a regularized regression problem where the regularizer penalizes at each point $\mbx$, the first-order deviation heavily along the direction where the \emph{covariance matrix} $\mbD(\mbx)$ is less spread, \ie the corresponding diffusivity is weak along that direction.

This perspective provides a connection to the problem of inducing anisotropic diffusion as a special instance of metric learning on Riemannian manifolds and, as the corresponding discretization, learning a graph structure from data. See~\cite{CarMah14} for an example of data-driven graph construction which relies on the known dimensionality of the underlying manifold.

\paragraph{On the convergence of $L^D$ to $\Delta^\calD$.}
It is well known that when data points $X$ are generated from an underlying Riemannian manifold $M$ embedded in an ambient Euclidean space, the isotropic graph Laplacian $L$ on $G=(X,E,W)$ converges to the Laplace-Beltrami operator $\Delta$ on $M$ as $n\to\infty$, with the neighborhood size $K\to\infty$ controlled accordingly~\cite{BelNiy05,HeiAudLix05}. However, despite its strong connection to the (continuous) anisotropic Laplacian $\Delta^\calD$ on $M$, our discrete anisotropic graph Laplacian $L^D$ is not by itself, \emph{consistent}, \ie it does not converge to $\Delta^\calD$ as $n\to \infty$. This is because, by design, our diffusivity operator is agnostic to the dimensionality $m$ of the manifold $M$. To elaborate this further, note that given fixed $n$-data points $X$ and the corresponding local neighborhood size $K$, our local diffusivity operator $D_i$ at $\mbx_i$ (Eq.~\ref{e:discdiffusivity}) defines a (new) inner-product in $H(E_i)$:
\begin{align}
D_i&:H(E_i)\to H(E_i) \nonumber\\
\Rightarrow \inner{\cdot,D_i \cdot}_{H(E_i)} &: H(E_i)\times H(E_i)\to \R.
\end{align}
The convergence of $L^D$ to $\Delta^\calD$ requires a certain form\footnote{Although $X\to M$ and $L\to \Delta$, the convergence of $H(E)$ to $\calT(M)$ cannot be uniquely defined (see~\cite{Hei05} for details) and therefore the convergence of $L^D$ (which depends on $D:H(E)\to H(E)$) to $\Delta^\calD$ is also not uniquely defined.} of convergence of $D^i$ to $\calD(
\mbx_i)$ at each $\mbx_i$. In particular, the continuum limit $D^\infty_i$ (as $n\to\infty$) of $D_i$ should induce an inner-product on $T_{\mbx_i}$. However, in general, $D^\infty_i$ cannot induce any inner product since $D^\infty_i$ has infinite degrees of freedom (\ie, $D^\infty_i$ has infinitely many parameters): $D_i$ has $K(n)$-eigenvalues and $K(n)\to\infty$ as $n\to\infty$. Actually, for a given fixed $n$ with corresponding $G_i$, $D_i$ can be defined as the restriction of $D^\infty_i$ on $E_i$. On the other hand, the continuous diffusivity operator $\calD$ on $T_{\mbx_i}$ has only up to $\frac{m(m+1)}{2}$-degrees of freedom with $m$ being the dimensionality of $M$. This implies that $D^\infty_i$ cannot be a bi-linear operator on $T_{\mbx_i}$. Actually, this is the only property that prevents $D^\infty_i$ being an inner-product: By construction, the limit of $D_i$ is non-negative and positive definite.

The relation between $\calD(\mbx_i)$ and $D_i$ is exactly the same as the relationship between the inner-product in the Euclidean space $\R^m$ and a nonlinear positive definite kernel $k:\R^m\times \R^m\to \R$ as commonly used in kernel machines: $k$ induces a similarity measure on $\R^m$. However, in general, it is not bi-linear and therefore it does not corresponds to an inner-product. Instead, $k$ induces an inner-product in a (potentially infinite-dimensional) feature space $\calK$ which is mapped by a nonlinear function $\phi:\R^m\to \calK$.

This insight leads to an algorithm to build \emph{consistent} local graph diffusivity operators $\{D^\calC_{i}\}$ (and the corresponding global operator $D^\calC$) by reducing the degree of freedom of each $D_i$ from $K(n)$ to $\frac{m(m+1)}{2}$. In the accompanying supplemental material, we show how $\{D^\calC_{i}\}$ can be explicitly constructed and it converges to $\calD$.

\paragraph{Discussion.}
While the consistent diffusivity operators might be of theoretical interest and may deserve further analysis, in this paper we focus on using the inconsistent diffusivity operator $D$ (Eq.~\ref{e:discdiffusivity}). This design choice is made based on two facts: 1) In general, estimating the dimensionality $m$ of a manifold $M$ and the corresponding tangent bundle $\calT(M)$ based on a finite sample $X\subset M$ are difficult problems~\cite{Keg02}. Therefore, existing approaches that involve estimating $m$ make it a hyper-parameter. Optimizing \emph{many} hyper-parameters is a difficult problem in semi-supervised learning due to the limited number of labeled points. 2) More importantly, some semi-supervised learning problems are inherently formulated as an inference on a graph $G$ that may not have any explicit connection to a manifold $M$ or the corresponding ambient space. For instance, if each node $\mbx_i\in X$ represents an image, and if each edge $e_{ij}\in E$ and corresponding weight $w_{ij}\in W$ represents the possibility of match and match score between $\mbx_i$ and $\mbx_j$, respectively, then there is no natural manifold or ambient space structure defined on $X$. Accordingly, our algorithm is obtained as a design choice that favors general applicability over theoretical consistency.

Lastly, we would like to add that it is tempting to build a consistency argument based on the fact that any graph with positive weights can be embedded into a manifold $M$ with a sufficiently high-dimensionality $m$, and therefore any data $X$ and the corresponding PD graph diffusivity operator $D$ can be regarded as a sample from such a manifold $M$ and the operators on $\calT(M)$, respectively. Unfortunately, this does not lead to a useful interpretation.

\section{Experiments}
\label{s:experiments}

\begin{table*}[t]
\begin{center}
\resizebox{\textwidth}{!}{
\begin{tabular}{l rrrrrrrrrrrr}
\toprule
Algorithm	& USPS 	& BCI & MNIST & COIL1 & COIL2 & RealSim & Pcmac & MPEG7 & SWDLEAF & ETH-80 & C-PASCAL & Avg.~\% \\
\midrule
$I$ 		&8.76  	&\ReRankTwo{41.60}    & 10.65 & 7.32	 & 4.37  & 23.61	 & 11.77  & 3.36  & \ReRankThree{2.39}    & 11.49 & 54.54 & 148.1\\ 
$A_{lin}$~\cite{SzlMagCoi08}	&5.55  &\ReRankThree{41.80}    & 8.47  & 7.36	 & 4.11  & 25.02 	 & 12.58  & 3.01  & 2.54    & 11.30 & 54.47 & 137.0\\ 
$A_{nlin}$         			&\ReRankThree{4.48} & \ReRankOne{39.53}   & \ReRankThree{7.62} &6.85    &2.98   & 23.46 &11.88  & 2.63  & 2.47    & \ReRankOne{9.91}      & \ReRankThree{52.22} & \ReRankThree{120.8}\\ 
$A_{LM}$ &\ReRankTwo{4.31}  &42.00    & \ReRankTwo{7.55}  & 6.48	 & \ReRankThree{2.22}  & \ReRankTwo{19.55}   & \ReRankTwo{11.47}  & \ReRankTwo{2.54}  & \ReRankOne{2.17}  	 & \ReRankThree{10.05} & \ReRankOne{51.19} & \ReRankTwo{111.7}\\ 
$A_{S}$ &\ReRankOne{3.93}  &42.13    &\ReRankOne{7.18}	 &\ReRankThree{6.21}	 &\ReRankTwo{2.13}  &\ReRankThree{20.08} 	 & \ReRankOne{11.34}& \ReRankThree{2.59}	 & \ReRankTwo{2.33}		 & \ReRankTwo{10.01} & \ReRankTwo{51.30} & \ReRankOne{110.5}\\ 
\midrule
GRF~\cite{ZhuGhaLaf03} &6.13 &42.68 &10.96 &\ReRankOne{4.93} &\ReRankOne{1.65} &28.09 &\ReRankThree{11.78} &2.96 &2.76 &12.16 &61.91 & 127.6\\
FLAP~\cite{GonTao14}	&5.66 &44.63 &10.99 &6.97 &2.73 &\ReRankThree{20.08} &14.49 &\ReRankOne{2.16} &2.84 &12.59 &57.97 & 131.0\\
LNP~\cite{WanZha06}		&7.27 &44.33 &13.25 &\ReRankTwo{5.53} &3.12 &\ReRankOne{16.02} &14.39 &N/A &N/A &11.94 &62.36 & 139.1\\
\bottomrule
\end{tabular}
}
\end{center}
\label{t:perfom}
\caption{Performance of different diffusion algorithms for semi-supervised learning: The three best results for each dataset are marked with boldface blue, plain green, and plain orange fonts, respectively. LNP~\cite{WanZha06} requires explicitly calculating the Euclidean distances between data points, and so it cannot be directly applied to \emph{MPEG7} and \emph{SWDLEAF} data sets. The final Avg.~\% column shows the mean percentage difference from the best result across all datasets, where 100\% would indicate that particular technique was best across all datasets.}
\end{table*}

We evaluate our anisotropic diffusion algorithm in classification on seven standard semi-supervised learning datasets~\cite{GuoNiuZha10,ZhoBel11,ChaSchZie10} and four object recognition datasets for which semi-supervised learning has been successful in the literature in retrieval contexts. We report performance for isotropic diffusion and the original kernel smoothing-type anisotropic diffusion approach of Szlam~\etal~\cite{SzlMagCoi08}. We also report the performances of three existing semi-supervised learning algorithms including Zhu~\etal's Gaussian random fields (\emph{GRFs})-based algorithm~\cite{ZhuGhaLaf03}, Gong and Tao's label propagation algorithm (\emph{FLAP}: Fick's Law Assisted Propagation,~\cite{GonTao14}) inspired by Fick's first law which describes the diffusion process at a steady state, and Wang and Zhang's~\cite{WanZha06} linear neighborhood propagation (\emph{LNP}) algorithm which automatically determines the edge weights $\{w_{ij}\}$ by representing each input point based on a convex combination of its neighbors~\cite{WanZha06}.

\paragraph{Datasets.} The \emph{MPEG7} shape dataset~\cite{LatLakEck00} consists of 1,400 images which show silhouettes of objects from 70 different categories. Adopting the experimental setting for data retrieval experiments \cite{DonBis13}, with 280 labels, we use shape matching~\cite{GopTurChe10} to infer pairwise distances from which the (isotropic) weight matrix $W$ is constructed. In this dataset, each data point $\mbx$ in $X$ is not explicitly presented and so the data generating manifold is not explicitly considered. Our algorithm is applicable even in this case, which justifies the use of the inconsistent diffusivity operator.\footnote{For consistent diffusivity operators, we would have to explicitly estimate the dimensionality of the data manifold; see Sec.~\ref{s:continuouscase}.}

The \emph{ETH-80} dataset consists of 3,280 photographs of objects from 8 different classes~\cite{LeiSch03}. The \emph{C-PASCAL} dataset (as a subset of the PASCAL VOC challenge 2008 data, where single objects are extracted based on bounding box annotations) contains 4,450 images of 20 classes~\cite{EbeLarSch10}. For both ETH-80 and C-PASCAL datasets, each data point is represented based on the HOG (histogram of oriented gradients) descriptors and the number of labels are set to 50 \cite{EbeFriSch12}. The \emph{SWDLEAF} (Swedish leaf) datasets contains 15 different tree species with 75 leaves per species~\cite{Sod01}. For this dataset, we use 50 labels per class, with Fourier descriptors to represent each entry~\cite{LinJac07}. 

\paragraph{Results.}
In Table~\ref{t:perfom}, $I$ refers to isotropic diffusion, $A_{lin}$ is the algorithm of Szlam~\etal~\cite{SzlMagCoi08}. $A_{nlin}$ is an extension of \cite{SzlMagCoi08} to nonlinear diffusion based on our diffusion approach (see Sec.~\ref{s:anidiffgraph}) while $A_{LM}$ and $A_{S}$ are \emph{local match} and \emph{smooth} anisotropic diffusion, respectively. 

Overall, all four anisotropic diffusion algorithms significantly improve classification accuracies over isotropic diffusion ($I$). However, for some datasets (\emph{SWDLEAF}, \emph{RealSim}, \emph{Pcmac}), the performance of linear anisotropic diffusion ($A_{lin}$)~\cite{SzlMagCoi08} is equal to or even worse than $I$. In contrast, all three nonlinear diffusion algorithms outperformed both $I$ and $A_{lin}$, while the \emph{local match} ($A_{LM}$) and \emph{smooth} ($A_{S}$) versions of the context-guided diffusion led to further improvement over $A_{nlin}$ in all but the \emph{ETH} and \emph{BCI} datasets. These results are in accordance with the superior performance of the smooth diffusivity operators (which is an example of exploiting context) in image processing and demonstrate the effectiveness of exploiting context information in anisotropic diffusion on graphs. For the \emph{BCI} dataset, $A_{nlin}$ and $A_{S}$ showed the best and the worst performances, while essentially all four anisotropic diffusion algorithms did not show any noticeable improvement from the isotropic case. This is because the initial labeling based on isotropic diffusion is almost random (around 40\% error rate for binary classification), and so this is a poor initialization for an anisotropic diffusion and does not lead to better label propagation. Similar observation were reported in \cite{SzlMagCoi08}. The anisotropic diffusion algorithms also demonstrated their competence in comparison with state-of-the-art label-propagation algorithms~\cite{ZhuGhaLaf03,GonTao14,WanZha06}: \emph{GRF} is best on \emph{COIL1} and \emph{COIL2}, and \emph{FLAP} and \emph{LNP} are the best for $\emph{MPEG7}$ and $\emph{RealSim}$. However, except for few cases, the results of $A_{nlin}$ and $A_{S}$ are included in the three best results for each dataset demonstrating the overall steady performance improvements over existing algorithms. Lastly, all three algorithms are designed for data graphs constructed based on input features rather than from function evaluations. Therefore, they can potentially benefit from our proposed anisotropic diffusion approaches.

\paragraph{Parameters.}
Isotropic diffusion has three parameters: the weight $\sigma_\mbx$ (Eq.~\ref{e:isoweights}), the size of local neighborhood $N_K$, and the number of diffusion steps $T$. We automatically determine $\sigma_\mbx$ based on the average Euclidean distance of $\mbx_j$ to $N_K(\mbx_j)$~\cite{SzlMagCoi08,HeiMai07}. We determine the two other parameters with a separate validation label set which is the same size as the training label set. 

For all anisotropic diffusion algorithms, an additional hyper-parameter $\sigma^2_f$ (Eq.~\ref{e:outputkernel}) is determined in the same way. The step size $\delta$ of the explicit Euler approximation in our algorithms (Eq.~\ref{e:graphanidiff}) is fixed at 1. In general, $\delta$ can also be tuned per dataset to improve performance. \emph{GRF}, \emph{FLAP}, and \emph{LNP} hyper-parameters are all determined in the same way based on the validation set.

\paragraph{Computational complexity.}
This depends upon the number $n$ of data points, the size $N_K$ of the local neighborhood, and the number of diffusion process iterations (Eq.~\ref{e:graphanidiff}). Each diffusion iteration requires multiplying the matrix $L^D$ of size $n\times n$ with a vector $f$ of size $n\times c$, where $c$ is the number of classes. Accordingly, in theory, the complexity of each step is $O(n^2c)$. However, typically $N_K \ll n$, which leads to a sparse matrix $L^D$: in practice, the computational complexity of each step is sub-quadratic. For USPS datasets with $1,500$ data points, running 100 iterations of the \emph{local match} diffusion process $A_{LM}$ takes $\approx$0.3 seconds on an Intel Xeon 3.4GHz CPU in MATLAB.

\section{Discussion and conclusion}
\label{s:discussion}
We show two ways to exploit local contexts: \emph{smooth} and \emph{local match}. These can be extended to consider the full topological features of $f$ evaluated at $E_i$ and $E_j$. For instance, one could perform spectral analysis on $W^D_i$ and $W^D_j$ and measure the similarity of the corresponding Eigenspectra to define a new diffusivity operator $D'$. This is different from pre-calculating topological features, as is commonly used in graph matching, since features are extracted from the input $X$ rather than from function evaluations $f$, and therefore the former stay constant during the diffusion process. We briefly explored this possibility in preliminary experiments, which indicate that full topological analysis is promising. However, due to the significantly increased computational complexity, we focus on \emph{smooth} and \emph{local match} operators and leave this extension for future work.

We adopted an explicit Euler scheme (Eq.~\ref{e:graphanidiff}) to discretize the continuous diffusion equation (Eq.~\ref{e:anilap}). This scheme can be obtained as a gradient descent step of the convex regularization functional $\calE$ (Eq.~\ref{e:anienergy}). An alternative implicit Euler scheme (Eq.~\ref{e:impeuler}) can be obtained as the analytic solution of $\calE$. Since our diffusion equation (Eq.~\ref{e:anilap}) is non-linear, both approaches eventually lead to iterative algorithms. A major advantage of an implicit Euler scheme is that it is uniformly stable with respect to $\delta$, while our explicit Euler scheme is stable only at sufficiently small values of $\delta$, which we regard as a hyper-parameter. On the other hand, implicit Euler approximation is computationally less favorable as it requires, at each iteration, explicitly solving a (sparse) linear system of size $n\times n$. Our explicit counterpart is computed by a matrix-vector multiplication. We choose the explicit scheme due to its fast convergence in experiments and its applicability to large-scale problems. Future work should carefully analyze the trade-off between these two approaches, especially on smaller-scale problems.

For simplicity of exposition, in Sec.~\ref{s:continuouscase}, we assumed that the underlying probability distribution $P$ on $M$ is uniform. However, our interpretation applies to more general cases where $P$ is non-uniform. If the sampling distribution $P$ on $M$ is non-uniform, the isotropic Laplace-Beltrami operator is locally weighted by the corresponding probability density $p$, rendering the \emph{weighted Laplacian}. In particular, if $p$ is differentiable, the weighted Laplacian is explicitly given as~\cite{HeiAudLix05,Gri13}:
\begin{align}
\Delta^p = \frac{1}{p}{\nabla^g}^*(p \nabla^g).
\end{align}
The weighted Laplacian satisfies Green's theorem, and the divergence theorem holds similarly~\cite{Gri06}. Accordingly, the corresponding weighted anisotropic Laplacian based on the diffusivity operator $\calD$ is obtained as in Proposition~\ref{pr:equivalance}.

\paragraph{Conclusion.}
We have presented an approach for anisotropic diffusion on graphs, by first extending well-established geometric diffusion on images to Riemannian manifolds and then discretizing it onto graphs. The resulting positive definite diffusivity operators on graphs leads to new diffusion possibilities that take local neighborhood structures into account, and thereby lead to robust diffusion. Applied to semi-supervised learning, our algorithms demonstrate improved accuracy over existing isotropic diffusion- and anisotropic diffusion-based algorithms.

\section*{Acknowledgements}
The authors thank the reviewers for their constructive feedback. Kwang In Kim thanks EPSRC EP/M00533X/1. James Tompkin and Hanspeter Pfister thank NSF CGV-1110955, the Air Force Research Laboratory, and the DARPA Memex program. Christian Theobalt thanks the Intel Visual Computing Institute.

{\small
\bibliographystyle{ieee}
\bibliography{./biblio}
}

\end{document}